\let\oldvec\vec
\documentclass{llncs}
\let\vec\oldvec
\usepackage{makeidx}  

\usepackage{verbatim}
\usepackage{graphicx}
\usepackage{amsmath}
\usepackage{amssymb}
\usepackage{mathabx}
\usepackage{array}
\usepackage{multirow}
\usepackage{placeins}

\newcommand{\Rcal}{\mathcal{R}}
\newcommand{\supp}{\operatorname{supp}}

\usepackage[usenames, dvipsnames]{color}

\usepackage{graphicx}
\graphicspath{{./Figures/}{./Computations/}}
\renewcommand{\qed}{\hfill$\square$}

\begin{document}
\title{Geometry of Policy Improvement} 
\titlerunning{Geometry of Policy Improvement}  
%
\author{Guido Mont\'ufar \and Johannes Rauh}
%
%
\tocauthor{Guido Mont\'ufar and Johannes Rauh}
\institute{	Max Planck Institute for Mathematics in the Sciences\\
	Inselstra\ss e 22, 04103 Leipzig, Germany
}
\maketitle              

\begin{abstract}
We investigate the geometry of optimal memoryless time independent decision making in relation to the amount of information that the acting agent has about the state of the system. We show that the expected long term reward, discounted or per time step, is maximized by policies that randomize among at most~$k$ actions whenever at most $k$ world states are consistent with the agent's observation. Moreover, we show that the expected reward per time step can be studied in terms of the expected discounted reward. Our main tool is a geometric version of the policy improvement lemma, which identifies a polyhedral cone of policy changes in which the state value function increases for all states. 
\end{abstract}

\section{Introduction}

We are interested in the amount of randomization that is needed in action selection mechanisms in order to maximize the expected value of a long term reward, depending on the uncertainty of the acting agent about the system state. 

It is known that in a Markov Decision Process (MDP), the optimal policy may always be chosen deterministic (see, e.g.,~\cite{Ross:1983:ISD:538843}), in the sense
that the action $a$ that the agent chooses is a deterministic function of the world state $w$ the agent observes. 
This is no longer true in a Partially Observable MDP (POMDP), where the agent does not observe $w$ directly, but only the value $s$ of a sensor.  In general, optimal memoryless policies for POMDPs are stochastic. 
However, the more information the agent has about~$w$, the less stochastic an optimal policy needs to be. 
As shown in~\cite{montufar2015geometry}, if a particular sensor value $s$ uniquely identifies~$w$, then the optimal policy may be chosen such that, on observing $s$, the agent always chooses the same action. 
We generalize this as follows: The agent may choose an optimal policy such that, 
if a given sensor value~$s$ can be observed from at most $k$ world states, then the agent chooses an action probabilistically among a set of at most~$k$ actions. 

Such characterizations can be used to restrict the search space when searching for an optimal policy. 
In~\cite{Ay2013selection}, it was proposed to construct a low-dimensional manifold of policies that contains an optimal policy in its closure and to restrict the learning algorithm to this manifold. 
In~\cite{montufar2015geometry}, it was shown how to do this in the POMDP setting when it is known that the optimal policy can be chosen deterministic in certain sensor states. 
This construction can be generalized and gives manifolds of even smaller dimension
when the randomization of the policy can be further restricted. 

As in~\cite{montufar2015geometry}, we study the case where at each time step the agent receives a reward that depends on the world state $w$ and the chosen action~$a$. 
We are interested in the long term reward in either the average or the discounted sense~\cite{suttonbarto98}. 
Discounted rewards are often preferred in theoretical analysis, because of the properties of the dynamic programming operators. 
In~\cite{montufar2015geometry}, the analysis of average rewards was much more involved than the analysis of discounted rewards. 
While the case of discounted rewards follows from a policy improvement argument, an elaborate geometric analysis was needed for the case of average rewards. 

Various works have compared average and discounted rewards~\cite{Tsitsiklis2002,Kakade2001,Hutter2006}. 
Here, we develop a tool that allows us to transfer properties of optimal policies from the discounted case to the
average case. 
Namely, the average case can be seen as the limit of the discounted case when the discount factor $\gamma$ approaches~$1$. 
If the Markov chain is irreducible and aperiodic, this limit is uniform, and the optimal policies of the discounted case converge to optimal policies of the average case.

\section{Optimal Policies for POMDPs}
\label{sec:notation}

A (discrete time) partially observable Markov decision process (POMDP) is defined by a tuple $(W,S,A,\alpha,\beta, R)$, 
where $W$, $S$, $A$ are finite sets of world states, sensor states, 
and actions, 
$\beta\colon W\to \Delta_S$ and $\alpha\colon W\times A\to \Delta_W$ are Markov kernels describing sensor measurements and world state transitions, 
and $R\colon W\times A\to \mathbb{R}$ is a reward signal. 
We consider stationary (memoryless and time independent) action selection mechanisms, described by Markov kernels of the form $\pi\colon S\to \Delta_A$. 
We denote the set of stationary policies by $\Delta_{S,A}$. 
We write $p^{\pi}(a| w) = \sum_{s}\beta(s| w)\pi(a|s)$ for the effective world state policy. 
%
Standard reference texts are~\cite{suttonbarto98,Ross:1983:ISD:538843}. 

We assume that the Markov chain starts with a distribution~$\mu\in\Delta_{W}$ and then progresses according to $\alpha$, $\beta$ and a fixed policy~$\pi$. 
We denote by~$\mu^{t}_{\pi}\in\Delta_{W}$ the distribution of the world state at time~$t$. 
It is well known that the limit $p^{\pi}_{\mu}:=\lim_{T\to\infty}\frac{1}{T}\sum_{t=0}^{T-1}\mu^{t}_\pi$ exists and is a stationary distribution of the Markov
chain. 
The following technical assumption is commonly made:
\begin{itemize}
	\item[$(*)$] For all $\pi$, the Markov chain over world states is aperiodic and irreducible.
\end{itemize}
The most important implication of irreducibility is that the limit distribution $p^{\pi}_{\mu}$ is independent of~$\mu$.
If the chain has period~$s$, then $p^{\pi}_{\mu} = \lim_{T\to\infty}\frac{1}{s}\sum_{t=1}^{s}\mu^{T+t}_{\pi}$.  In
particular, under assumption~$(*)$, $\mu^{t}_\pi\to p^{\pi}_{\mu}$ for any~$\mu$.  (Since we assume finite sets, all notions of convergence of probability distributions are equivalent.)

The objective of learning is to maximize the expected value of a long term reward. 
The (normalized) discounted reward with discount factor $\gamma\in[0,1)$ is
\begin{equation*}
  \!\!
  \Rcal_\mu^\gamma(\pi)
  = (1-\gamma) \sum_{t=0}^{\infty} \gamma^{t}\sum_{w}\mu^{t}_{\pi}(w) \sum_{a}
  p^{\pi}(a| w)
  R(a,w)  
    = (1-\gamma) \mathbb{E}_{\pi,\mu}\Big[\sum_{t=0}^{\infty} \gamma^{t} R(a_{t},w_{t})\Big]. 
  \!\!
\end{equation*}
The average reward is
\begin{equation*}
\Rcal_\mu(\pi)
= \sum_{w} p^{\pi}_{\mu}(w) \sum_{a} 
p^{\pi}(a| w)
R(a,w).
\end{equation*}
Under assumption $(*)$, $\Rcal_{\mu}$ is independent of the choice of~$\mu$ and
depends continuously on~$\pi$, as we show next. 
Since $\Delta_{S,A}$ is compact, the existence of optimal policies is guaranteed. 
Without the assumption, optimal policies need not exist. 
On the other hand, the expected discounted reward $\Rcal_{\gamma}^{\mu}$ is always continuous, so that, for this, optimal policies always exist.

\begin{lemma}
	\label{lem:R-continuous}
	Under assumption~$(*)$, $\Rcal_\mu(\pi)$ is continuous as a function of~$\pi$.
\end{lemma}
\begin{proof}
By $(*)$, $p^\pi_{\mu}$ is the unique solution to a linear system of equations that smoothly depends on~$\pi$.  
  Thus, $\Rcal_\mu$ is continuous as a function of~$\pi$.
  \qed\end{proof}
\begin{lemma}
	\label{lem:Rgamma-continuous}
	For fixed~$\mu$ and $\gamma\in[0,1)$, $\Rcal_{\gamma}^{\mu}(\pi)$ is continuous as a function of~$\pi$.
\end{lemma}
\begin{proof}
  Fix $\epsilon>0$.  There exists $l>0$ such that $(1-\gamma)\sum_{t=l}^{\infty}\gamma^{t} R \le \epsilon/4$, where
  $R=\max_{a,w}|R(a,w)|$.  For each~$t$, the distribution $\mu_{\pi}^{t}$ depends continuously on~$\pi$.  For fixed
  $\pi$, let $U$ be a neighborhood of~$\pi$ such that $|\mu_{\pi}^{t}(w) - \mu_{\pi'}^{t}(w)| \le
  \frac{1}{2|W|R}\epsilon$ for $t=0,\dots,l-1$ and~$\pi'\in U$.  Then, for all $\pi'\in U$,
	\begin{equation*}
	|\Rcal_{\gamma}^{\mu}(\pi) - \Rcal_{\gamma}^{\mu}(\pi')|
	\le \frac{\epsilon}{2} + (1-\gamma) \sum_{t=0}^{l-1} \gamma^{t} \sum_{w} |\mu_{\pi}^{t}(w) - \mu_{\pi'}^{t}(w)| R
	\le \frac{\epsilon}{2} + \frac{|W|}{2 |W|R} \epsilon R = \epsilon. 
    \tag*{$\square $}
	\end{equation*}
\end{proof}

The following refinement of the analysis of~\cite{montufar2015geometry} is our main result.
\begin{theorem}
  \label{thm:bounded-stochasticity}
  Consider a POMDP $(W,S,A,\alpha,\beta,R)$, and let $\mu\in\Delta_{W}$ and $\gamma\in[0,1)$. 
 There is a stationary (memoryless, time independent) policy $\pi^\ast\in \Delta_{S,A}$ with
  $|\supp(\pi^\ast(\cdot|s))| \leq |\supp(\beta(s|\cdot))|$ for all $s\in S$ and $\Rcal_{\mu}^{\gamma}(\pi^\ast)\geq
  \Rcal_{\mu}^{\gamma}(\pi)$ for all $\pi\in \Delta_{S,A}$.
  Under assumption~$(*)$, the same holds true for~$\Rcal_{\mu}$ in place of~$\Rcal_{\mu}^{\gamma}$. 
\end{theorem}

We prove the discounted case in Section~\ref{sec:disc-rewards} and the average case in Section~\ref{sec:averagefromdiscounted}.

\section{Discounted Rewards from Policy Improvement}
\label{sec:disc-rewards}
The state value function $V^\pi$ of a policy $\pi$ is defined as the unique solution of the Bellman equation 
\begin{equation*}
  V^\pi(w) = \sum_a p^\pi(a|w)\Big[R(w,a) +  \gamma \sum_{w'}\alpha(w'|w,a) V^\pi(w')\Big],\quad w\in W. 
\end{equation*}
It is useful to write $V^\pi(w) = \sum_a p^\pi(a|w) Q^{\pi}(w,a)$, where
\begin{equation*}
  \textstyle
  Q^\pi(w,a) = R(w,a) + \gamma \sum_{w'}\alpha(w'|w,a)V^\pi(w'), \quad w\in W, a\in A, 
\end{equation*}
is the state action value function. 
Observe that $\Rcal_{\mu}^\gamma(\pi)=(1-\gamma)\sum_w\mu(w)V^\pi(w)$. 
If two policies $\pi,\pi'$ satisfy $V^{\pi'}(w)\geq V^\pi(w)$ for all $w$, then $\Rcal_{\mu}^\gamma(\pi')\geq
\Rcal_{\mu}^\gamma(\pi)$ for all~$\mu$.
The following is a more explicit version of a lemma from~\cite{montufar2015geometry}: 
\begin{lemma}[Policy improvement lemma]
	\label{lemma:policyimprovement}
	Let $\pi, \pi' \in \Delta_{S,A}$ and 
$\epsilon(w) = \sum_a p^{\pi'}(a|w) Q^\pi(w,a) - V^\pi(w)$ for all $w\in W$. 
Then
    \begin{equation*}
      V^{\pi'}(w) = V^{\pi}(w) + \mathbb{E}_{\pi',w_0=w}\Big[\sum_{t=0}^{\infty}\gamma^t\epsilon(w_t)\Big]\quad \text{for all }w\in W.
    \end{equation*}
    If $\epsilon(w)\geq0$ for all $w\in W$, then
    \begin{equation*}
      V^{\pi'}(w)\geq V^{\pi}(w) + d^{\pi'}(w) \epsilon(w)\quad\text{for all $w\in W$},
    \end{equation*}
    where $d^{\pi'}(w) = \sum_{t=0}^\infty \gamma^t \Pr(w_t=w|\pi',w_0=w)\ge 1$ is the discounted expected number of visits
    to~$w$.
\end{lemma}
\begin{proof}
  $\displaystyle V^\pi(w) = \sum_a p^{\pi'}(a|w)Q^\pi(w,a) - \epsilon(w)$
  \setlength{\abovedisplayskip}{0pt}
 \begin{align*}
    &= \mathbb{E}_{\pi',w_0=w}\Big[\Big(R(w_0, a_0) - \epsilon(w_0)\Big) + \gamma V^\pi(w_1) \Big]\\
    &= \mathbb{E}_{\pi',w_0=w}\Big[\Big(R(w_0, a_0)- \epsilon(w_0)\Big) + \gamma \Big(\sum_{a}p^{\pi'}(a|w_1)Q^\pi(w_1,a) - \epsilon(w_1)\Big) \Big]\\
    &= \mathbb{E}_{\pi',w_0=w}\Big[\sum_{t=0}^{\infty} \gamma^t \Big(R(w_t, a_t)  - \epsilon(w_t)\Big)\Big] 
   =  V^{\pi'}(w) - \mathbb{E}_{\pi',w_0=w}\Big[\sum_{t=0}^{\infty}\gamma^t \epsilon(w_t)\Big].
    \tag*{$\square $}
  \end{align*}
\end{proof}

Lemma~\ref{lemma:policyimprovement} allows us to find policy changes that increase $V^{\pi}(w)$ for all~$w\in W$ and thereby $\Rcal^{\gamma}_{\mu}(\pi)$ for any~$\mu$.
\begin{definition}
  Fix a policy $\pi\in\Delta_{S,A}$. For each sensor state $s\in S$ consider the set $\supp(\beta(s|\cdot))=\{w\in W\colon\beta(s|w)>0 \} =
  \{ w^{s}_1,\ldots, w^{s}_{k_{s}}\}$, and define the linear forms
  \begin{equation*}
    l_i^{\pi,s}\colon \Delta_A\to \mathbb{R};\;q\mapsto \sum_a q(a) Q^\pi(w^s_i,a), \quad i=1,\ldots,k_{s}. 
  \end{equation*}
  The \emph{policy improvement cone} at policy~$\pi$ and sensation $s$ is
  \begin{equation*}
    L^{\pi,s} = \big\{ q\in\Delta_{A} : l_{i}^{\pi,s}(q) \ge l_{i}^{\pi,s}(\pi(\cdot|s)) \text{ for all }i=1,\ldots, k_s \big\}.
  \end{equation*}
  The \emph{(total) policy improvement cone} at policy~$\pi$ is
  \begin{equation*}
    L^{\pi} = \big\{ \pi'\in\Delta_{S,A} : \pi'(\cdot|s)\in L^{\pi,s}\text{ for all }s\in S \big\}.
  \end{equation*}
\end{definition}
$L^{\pi,s}$ and $L^{\pi}$ are intersections of $\Delta_{A}$ and~$\Delta_{S,A}$ with intersections of affine halfspaces (see Fig.~\ref{figure:1}). 
Since $\pi\in L^{\pi}$, the policy improvement cones are never empty. 

\begin{lemma}
  \label{lem:policy-improvement-cone}
Let $\pi\in\Delta_{S,A}$ and $\pi'\in L^{\pi}$. Then, for all~$w$,
  \begin{equation*}
    V^{\pi'}(w) - V^\pi(w) \geq  d^{\pi'}(w) \sum_s \beta(s|w) \sum_{a} (\pi'(a|s) - \pi(a|s)) Q^\pi(w,a) \ge 0. 
  \end{equation*}
\end{lemma}
\begin{proof}
  Fix $w\in W$.  In the notation from Lemma~\ref{lemma:policyimprovement}, suppose that $\supp(\beta(\cdot|w)) =
  \{s_{1},\dots,s_{l}\}$ and that $w=w_{i_{j}}^{s_{j}}$ for $j=1\dots,l$.  Then
  \begin{align*}
    \epsilon(w)
    & = \sum_a p^{\pi'}(a|w) Q^\pi(w,a) - \sum_{a} p^{\pi}(a|w) Q^{\pi}(w,a) \\
&
   = \sum_{j=1}^{l}\beta(s_{j}|w) l^{\pi,s_{j}}_{i_{j}}(\pi'(\cdot|s_j) - \pi(\cdot|s_j)) \ge 0,
  \end{align*}
  since $\pi'\in L^\pi$. 
  The statement now follows from Lemma~\ref{lemma:policyimprovement}. 
\end{proof}
\begin{remark}
Lemma~\ref{lem:policy-improvement-cone} relates to the policy gradient theorem~\cite{Sutton00policygradient}, which says that 
\begin{equation}
\frac{\partial V^\pi(w)}{\partial\pi(a'|s')} = d^\pi(w) \sum_s \beta(s|w)\sum_a \frac{\partial \pi(a|s)}{\partial \pi(a'|s')} Q^\pi(w,a). 
\end{equation}
Our result adds that, for each $w$, the value function $V^{\pi'}(w)$ is bounded from below by a linear function of $\pi'$ that takes value at least $V^\pi(w)$ within the entire policy improvement cone $L^\pi$. 
See Fig.~\ref{figure:1}. 
\end{remark}
Now we show that there is an optimal policy with small support. 
\begin{lemma}
  \label{lem:dimension-lemma}
  Let $P$ be a polytope, and let $l_{1},\dots,l_{k}$ be linear forms on~$P$. 
  For any $p\in P$, let
  $L_{i,+} = \{q\in P\colon l_{i}(q) \ge l_{i}(p) \}$. 
  Then $\bigcap_{i=1}^{k}L_{i,+}$ contains an element $q$ that belongs to a face of~$P$ of dimension at most~$k-1$.
\end{lemma}
\begin{proof}
  The argument is by induction. 
  For~$k=1$, the maximum of $l_{1}$ on~$P$ is attained at a vertex~$q$
  of~$P$. Clearly, $l_{1}(q)\ge l_{1}(p)$, and so~$q\in L_{1,+}$. 

  Now suppose that~$k>1$. 
  Let $P' := P\cap L_{k,+}$. 
  Each face of~$P'$ is a subset of a face of~$P$ of at most one more dimension. 
  By induction, $\bigcap_{i=1}^{k-1}L_{i,+}\cap P'$ contains an element $q$ that belongs to a face of~$P'$ of dimension at most~$k-2$. 
\end{proof}

\begin{proof}[of Theorem~\ref{thm:bounded-stochasticity} for discounted rewards] 
By Lemma~\ref{lem:dimension-lemma}, each policy improvement cone $L^{\pi,s}$ contains an element $q$ that belongs to a face of $\Delta_{A}$ of dimension at most~$(k-1)$ (that is, the support of $q$ has cardinality at most~$k$), where $k = |\supp(\beta(s|\cdot))|$. 
Putting these together, we find a policy $\pi'$ in the total policy improvement cone that satisfies
$|\supp(\pi(\cdot|s))|\le|\supp(\beta(s|\cdot))|$ for all $s$.  By Lemma~\ref{lem:policy-improvement-cone}, $V^{\pi'}(w)\ge V^{\pi}(w)$ for all~$w$, and so $\Rcal_{\mu}^{\gamma}(\pi')\ge\Rcal_{\mu}^{\gamma}(\pi)$.
\qed
\end{proof}	

 \begin{figure}
 	\centering
 	\includegraphics[clip=true,trim=0cm .3cm .25cm 0cm, scale=.8]{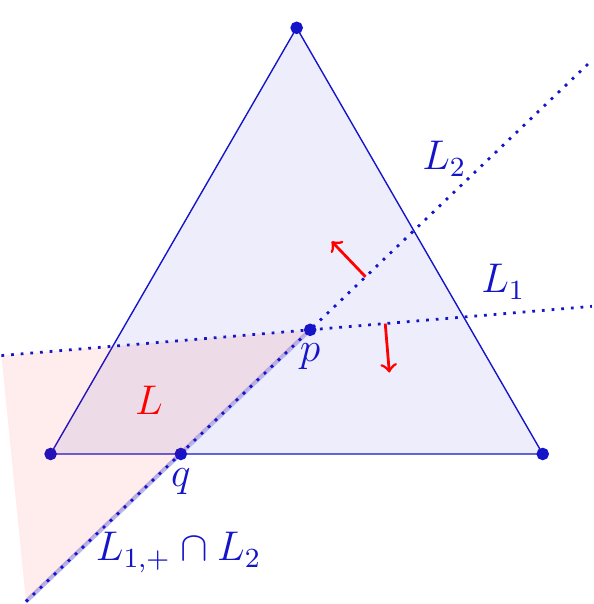}\qquad\qquad
 	\includegraphics[clip=true,trim=0cm 0cm 0cm 0cm, width=4.6cm]{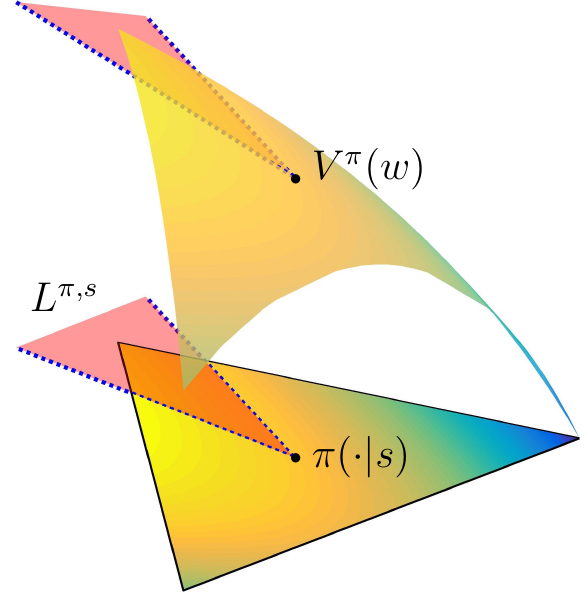}
 	\caption{
Left: Illustration of the policy improvement cone. 
Right: Illustration of the state value function $V^\pi(w)$ for some fixed $w$, 
showing the linear lower bound over the policy improvement cone $L^{\pi,s}$. 
This numerical example is discussed further in Section~\ref{sec:example}. }
\label{figure:1}
 \end{figure}

\begin{remark}
The $|\supp \beta(s|\cdot)|$ positive probability actions at sensation $s$ do not necessarily correspond to the actions that the agent would choose if she knew the identity of the world
state, as shown in our example from Section~\ref{sec:example}. 
\end{remark}

\section{Average Rewards from Discounted Rewards}
\label{sec:averagefromdiscounted}

The average reward per time step can be written in terms of the discounted reward as $\Rcal(\pi)=\Rcal_{p^{\pi}_{\mu}}^{\gamma}$. 
However, the hypothesis $V^{\pi'}(w)\geq V^\pi(w)$ for all $w$, does not directly imply any relation between $\Rcal(\pi')$ and $\Rcal(\pi)$, since they compare the value function against different stationary distributions. 
We show that results for discounted rewards translate nonetheless to results for average rewards.

\begin{lemma}
	\label{lem:mu-converges-uniformly}
	Let $\mu$ be fixed, and assume~$(*)$.  For any $\epsilon>0$ there exists $l>0$ such that for all $\pi$ and all $t\ge l$, 
$|\mu_{\pi}^{t}(w) - p^{\pi}_{\mu}(w)| \le \epsilon$ for all $w$.
\end{lemma}
\begin{proof}
%
  By~$(*)$, the transition matrix of the Markov chain has the eigenvalue one with multiplicity one, with left
  eigenvector is $p^{\pi}_{\mu}$.  Let $p_{2},\dots,p_{|W|}$ be orthonormal left eigenvectors to the other eigenvalues
  $\lambda_{2},\dots,\lambda_{|W|}$, 
  ordered such that $\lambda_2$ has the largest absolute value. 
  There is a unique expansion $\mu = c_{1} p^{\pi}_{\mu} + c_{2}p_{2} + \dots + c_{|W|}p_{|W|}$.  Then $\mu_{\pi}^{t} = c_{1}
  p^{\pi}_{\mu} + \sum_{i=2}^{|W|} c_{i} \lambda_{i}^{t} p_{i}$.  Letting $t\to\infty$, it follows that $c_{1} = 1$.  By
  orthonormality, $|c_{i}|^{2}\le \sum_{i=2}^{|W|}c_{i}^{2} \le \|\mu\|_{2}^{2} \le 1$ and $|p_{i}(w)|\le 1$ for
  $i=2,\dots,|W|$.  Therefore,
    $|\mu_{\pi}^{t}(w) - p^{\pi}_{\mu}(w)| = |\sum_{i=2}^{|W|} c_{i} \lambda_{i}^{t} p_{|W|}(w)|
    \le |W| |\lambda_{2}|^{t}$.
	
	Since $|\lambda_{2}|$ depends continuously on the transition matrix, which depends continuously on $\pi$,
    $|\lambda_{2}|$ depends continuously on~$\pi$.  Since $\Delta_{S,A}$ is compact, $|\lambda_{2}|$ has a maximum~$d$, and $d<1$ due
    to~$(*)$. 
    Therefore, $|\mu_{\pi}^{t}(w) - p^{\pi}_{\mu}(w)|\le |W| d^{t}$ for all~$\pi$.  The statement follows from this.
    \qed\end{proof}

\begin{proposition}
	\label{prop:R-converges-uniformly}
	For fixed~$\mu$, 
	under assumption~$(*)$, $\Rcal_\mu^\gamma(\pi)\to \Rcal_\mu(\pi)$ uniformly in~$\pi$ as $\gamma\to1$. %
	Thus, $\Rcal_\mu^\gamma\to \Rcal_\mu$ uniformly in $\pi$ as $\gamma\to1$.
\end{proposition}
\begin{proof}
  For fixed $\mu$ and $\epsilon$, let $l$ be as in Lemma~\ref{lem:mu-converges-uniformly}.  Let
  $R=\max_{a,w}|R(a,w)|$.~Then
	\begin{align*}
	\Rcal_\mu^\gamma(\pi)
	& = (1-\gamma)\sum_{k=0}^{l-1} \gamma^{k} \sum_{w}\mu^{k}_{\pi}(w) \sum_{a}\pi(a|w) R(a,w) \\
	& \quad
	+ (1-\gamma) \gamma^{l} \sum_{k=0}^{\infty}\gamma^{k} \sum_{w}p_{\mu}^{\pi}(w) \sum_{a}\pi(a|w) R(a,w)
	+ O(\epsilon R) (1-\gamma)\sum_{k=0}^{\infty}\gamma^{k} \\
	& = O((1-\gamma) l R) + O(\epsilon R) + \gamma^{l} \Rcal_\mu(\pi)
	\end{align*}
	for all~$\pi$.
	For given $\delta>0$, we can choose $\epsilon>0$ such that the term $O(\epsilon R)$ is smaller in absolute value
	than~$\delta/3$.  This also fixes $l=l(\epsilon)$.  Then, for any $\gamma<1$ large enough, the term $O((1-\gamma)l R)$
	is smaller than $\delta/3$, and also $|(\gamma^{l} - 1) \Rcal_\mu(\pi)|\le \delta/3$.  This shows that for $\gamma<1$
	large enough, $|\Rcal_\mu^\gamma(\pi) - \Rcal_\mu(\pi)| \le \delta$, independent of~$\pi$.  The statement follows since
	$\delta>0$ was arbitrary.
\qed\end{proof}

\begin{theorem}
	\label{thm:maximum-convergence}
	For any~$\gamma\in[0,1)$, let $\hat\pi_{\gamma}$ be a policy that maximizes~$\Rcal_{\gamma}^{\mu}$.  Let $\hat\pi$ be a
	limit point of a convergent subsequence as $\gamma\to 1$.  Then $\hat\pi$ maximizes~$\Rcal_\mu$, and
	$\lim_{\gamma\to1}\Rcal^{\gamma}_{\mu}(\hat\pi_{\gamma}) = \Rcal_\mu(\hat\pi)$.
\end{theorem}
\begin{proof}
	For any $\epsilon>0$, there is $\delta>0$ such that $\gamma\ge 1-\delta$ implies $|\Rcal_\mu(\pi) -
	\Rcal_\mu^\gamma(\pi)|\le\epsilon$ for all~$\pi$. 
	Thus $|\max_{\pi} \Rcal_\mu(\pi) - \max_{\pi}\Rcal_\mu^\gamma|\le\epsilon$, whence
	$\lim_{\gamma\to 1}\max_{\pi}\Rcal_\mu^\gamma(\pi) = \max_{\pi}\Rcal_\mu(\pi)$.
	Moreover, 
	$|\max_{\pi}\Rcal_\mu(\pi) - \Rcal_\mu(\hat\pi_{\gamma})|
	\le 2\epsilon + |\max_{\pi}\Rcal_\mu^\gamma(\pi) - \Rcal_\mu^\gamma(\hat\pi_{\gamma})| = 2\epsilon$. 
	By continuity, the limit value of $\Rcal_\mu$ applied to a convergent subsequence of the $\hat\pi_{\gamma}$ is the
	maximum of~$\Rcal_\mu$.
\qed\end{proof}

\begin{corollary}
	\label{cor:support-restrictions-survive}
	Fix a world state~$w$, and let~$r\ge 0$. 
	If there exists for each~$\gamma\in[0,1)$ a policy $\hat\pi_{\gamma}$ that
	is optimal for $\Rcal_\mu^\gamma$ with $|\supp(\pi(\cdot|s))| \le r$, then there exists a policy $\hat\pi$ with
	$|\supp(\pi(\cdot|s))| \le r$ that is optimal for~$\Rcal_\mu$. 
\end{corollary}
\begin{proof}
	Take a limit point of the family $\hat\pi_{\gamma}$ as $\gamma\to1$ and apply Theorem~\ref{thm:maximum-convergence}.
\qed\end{proof}

\begin{remark}
  Without~$(*)$, one can show that $\Rcal^{\gamma}_{\mu}(\pi)$ still converges to $\Rcal_{\mu}(\pi)$ for each
  fixed~$\pi$, but convergence is no longer uniform.
  Also, $\Rcal_{\mu}$ need not be continuous in~$\pi$, and so an optimal policy need not exist. 
\end{remark}

\section{Example}
\label{sec:example}

We illustrate our results on an example from~\cite{montufar2015geometry}.
Consider an agent with sensor states $S=\{1,2,3\}$ and actions $A=\{1,2,3\}$. 
The system has world states $W=\{1,2,3,4\}$ with the transitions and rewards illustrated in Fig.~\ref{fig:example}. 
At $w=1,4$ all actions produce the same outcomes. States $w=2,3$ are observed as $s=2$. 
Hence we can focus on $\pi(\cdot|s=2)\in\Delta_A$. 
We evaluate $861$ evenly spaced policies in this $2$-simplex. 
Fig.~\ref{fig:example} shows color maps of the expected reward (interpolated between evaluations), with lighter colors corresponding to higher values. 
As in Fig.~\ref{figure:1}, red vectors are the gradients of the linear forms (corresponding to $Q^\pi(w,\cdot)$, $w=2,3$), and dashed blue lines limit the policy improvement cones $L^{\pi,s=2}$.
Stepping into the improvement cone always increases $V^\pi(w) =\Rcal^\gamma_{\mu=\delta_w}(\pi)$ for all $w\in W$. 
Note that each cone contains a policy at an edge of the simplex, i.e., assigning positive probability to at most two actions. 
The convergence of $\Rcal^\gamma_\mu$ to $\Rcal_\mu$ as $\gamma\to1$ is visible. 
Note also that for $\gamma=0.6$ the optimal policy requires two positive probability actions, so that our upper bound $|\supp(\pi(\cdot|s))|\leq |\supp(\beta(s|\cdot))|$ is attained. 

\begin{figure}[t]
\centering	
\includegraphics[scale=.8]{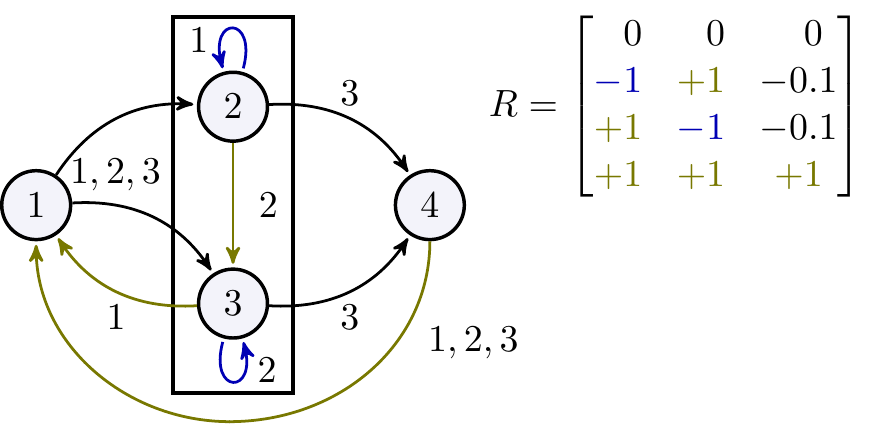}\\
\includegraphics[clip=true,trim=0cm 80cm 0cm 0cm, width=11cm]{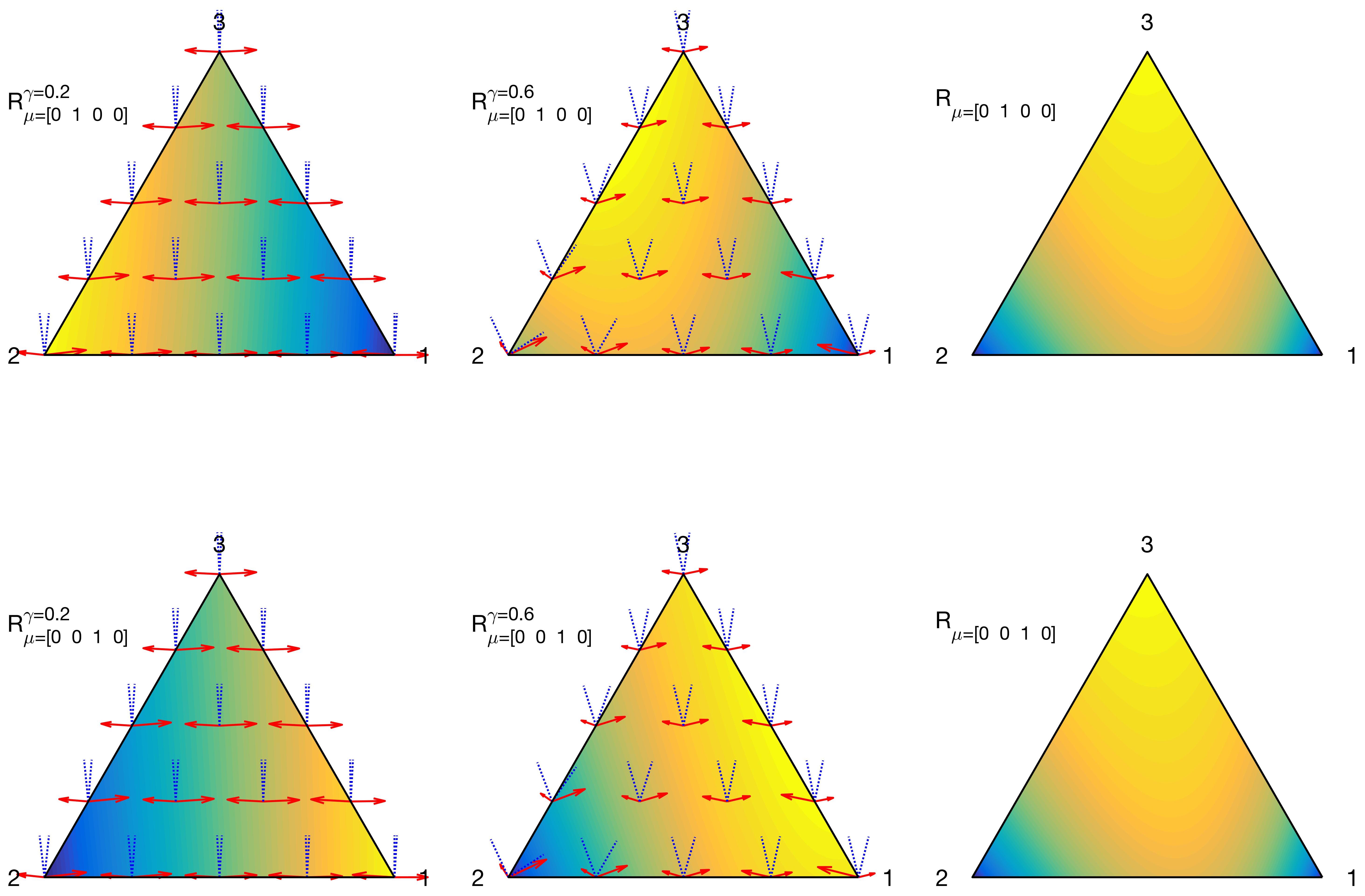}\\
\includegraphics[clip=true,trim=0cm 0cm 0cm 80cm, width=11cm]{example1dirs}
\caption{Illustration of the example form Section~\ref{sec:example}. Top: State transitions and reward signal. 
		Bottom: Numerical evaluation of the expected long term reward. }
	\label{fig:example}
\end{figure}

\medskip
\noindent
{\bf Acknowledgment:} We thank Nihat Ay for support and insightful comments. 

\FloatBarrier


\bibliography{referenzen}

\begin{thebibliography}{1}

\bibitem{Ay2013selection}
N.~Ay, G.~Mont{\'u}far, and J.~Rauh.
\newblock {\em Advances in Cognitive Neurodynamics (III)}, chapter Selection
  Criteria for Neuromanifolds of Stochastic Dynamics, pages 147--154.
\newblock Springer Netherlands, 2013.

\bibitem{Hutter2006}
M.~Hutter.
\newblock General discounting versus average reward.
\newblock In {\em Algorithmic Learning Theory 17}, pages 244--258. Springer
  Berlin Heidelberg, 2006.

\bibitem{Kakade2001}
S.~Kakade.
\newblock Optimizing average reward using discounted rewards.
\newblock In {\em Computational Learning Theory 14}, pages 605--615. Springer
  Berlin Heidelberg, 2001.

\bibitem{montufar2015geometry}
G.~Mont\'ufar, K.~Ghazi-Zahedi, and N.~Ay.
\newblock Geometry and determinism of optimal stationary control in partially
  observable {M}arkov decision processes.
\newblock {\em arXiv:1503.07206}, 2015.

\bibitem{Ross:1983:ISD:538843}
S.~M. Ross.
\newblock {\em Introduction to Stochastic Dynamic Programming}.
\newblock Academic Press, Inc., 1983.

\bibitem{suttonbarto98}
R.~S. Sutton and A.~G. Barto.
\newblock {\em Reinforcement Learning: An Introduction}.
\newblock MIT Press, 1998.

\bibitem{Sutton00policygradient}
R.~S. Sutton, D.~{McAllester}, S.~Singh, and Y.~Mansour.
\newblock Policy gradient methods for reinforcement learning with function
  approximation.
\newblock In {\em Advances in Neural Information Processing Systems 12}, pages
  1057--1063. MIT Press, 2000.

\bibitem{Tsitsiklis2002}
J.~N. Tsitsiklis and B.~Van~Roy.
\newblock On average versus discounted reward temporal-difference learning.
\newblock {\em Machine Learning}, 49(2):179--191, 2002.

\end{thebibliography}
\bibliographystyle{abbrv}

\end{document}